\documentclass{article}

     \usepackage[nonatbib,final]{neurips_2022}

\usepackage[utf8]{inputenc} 
\usepackage[T1]{fontenc}    
\usepackage{hyperref}       
\usepackage{url}            
\usepackage{booktabs}       
\usepackage{amsfonts}       
\usepackage{nicefrac}       
\usepackage{microtype}      
\usepackage{xcolor}         
\usepackage{graphicx}
\usepackage{subfigure}
\usepackage{amsmath}
\usepackage{amssymb}
\usepackage{multirow}
\usepackage{makecell}
\usepackage{bbding}
\usepackage{wrapfig}
\usepackage{listings}
\usepackage{algorithm}
\usepackage{algorithmicx}
\usepackage{algpseudocode} 
\usepackage{amsthm}
\newcommand{\eg}{\emph{e.g.}}
\newcommand{\etal}{\emph{et~al.}}
\newcommand{\ie}{\emph{i.e.}}

\newtheorem{myTheo}{Theorem}

\makeatletter
\newenvironment{breakablealgorithm}
{
	\begin{center}
		\refstepcounter{algorithm}
		\hrule height.8pt depth0pt \kern2pt
		\renewcommand{\caption}[2][\relax]{
			{\raggedright\textbf{\ALG@name~\thealgorithm} ##2\par}%
			\ifx\relax##1\relax 
			\addcontentsline{loa}{algorithm}{\protect\numberline{\thealgorithm}##2}%
			\else 
			\addcontentsline{loa}{algorithm}{\protect\numberline{\thealgorithm}##1}%
			\fi
			\kern2pt\hrule\kern2pt
		}
	}{
		\kern2pt\hrule\relax
	\end{center}
}
\makeatother

\title{Vision GNN: An Image is Worth Graph of Nodes}

\author{%
  Kai Han$^{1,2*}$\quad Yunhe Wang$^{2}$\thanks{Equal contribution.}~\quad Jianyuan Guo$^{2}$\quad Yehui Tang$^{2,3}$\quad Enhua Wu$^{1,4}$\\
  $^1$State Key Lab of Computer Science, ISCAS \& UCAS\\
  $^2$Huawei Noah's Ark Lab\\
  $^3$Peking University\hspace{0.2cm} $^4$University of Macau\\
  \texttt{\{kai.han,yunhe.wang\}@huawei.com, weh@ios.ac.cn} \\
}

\begin{document}

\maketitle

\begin{abstract}
  Network architecture plays a key role in the deep learning-based computer vision system. The widely-used convolutional neural network and transformer treat the image as a grid or sequence structure, which is not flexible to capture irregular and complex objects. In this paper, we propose to represent the image as a graph structure and introduce a new \emph{Vision GNN} (ViG) architecture to extract graph-level feature for visual tasks. We first split the image to a number of patches which are viewed as nodes, and construct a graph by connecting the nearest neighbors. Based on the graph representation of images, we build our ViG model to transform and exchange information among all the nodes. ViG consists of two basic modules: Grapher module with graph convolution for aggregating and updating graph information, and FFN module with two linear layers for node feature transformation. Both isotropic and pyramid architectures of ViG are built with different model sizes. Extensive experiments on image recognition and object detection tasks demonstrate the superiority of our ViG architecture. We hope this pioneering study of GNN on general visual tasks will provide useful inspiration and experience for future research.
  
  The PyTorch code is available at \url{https://github.com/huawei-noah/Efficient-AI-Backbones} and the MindSpore code is available at \url{https://gitee.com/mindspore/models}.
\end{abstract}

\section{Introduction}
\label{sec:intro}
In the modern computer vision system, convolutional neural networks (CNNs) used to be the de-facto standard network architecture~\cite{lecun1998gradient,alexnet,resnet}. Recently, transformer with attention mechanism was introduced for visual tasks~\cite{vit,detr} and attained competitive performance. MLP-based (multi-layer perceptron) vision models~\cite{mixer,resmlp} can also work well without using convolutions or self-attention. These progresses are pushing the vision models towards an unprecedented height.

Different networks treat the input image in different ways. As shown in Figure~\ref{Fig:image}, the image data is usually represented as a regular grid of pixels in the Euclidean space. CNNs~\cite{lecun1998gradient} apply sliding window on the image and introduce the shift-invariance and locality. The recent vision transformer~\cite{vit} or MLP~\cite{mixer} treats the image as a sequence of patches. For example, ViT~\cite{vit} divides a $224\times 224$ image into a number of $16\times 16$ patches and forms a sequence with length of $196$ as input.

Instead of the regular grid or sequence representation, we process the image in a more flexible way. One basic task of computer vision is to recognize the objects in an image. Since the objects are usually not quadrate whose shape is irregular, the commonly-used grid or sequence structures in previous networks like ResNet and ViT are redundant and inflexible to process them. An object can be viewed as a composition of parts, \eg, a human can be roughly divided into head, upper body, arms and legs. These parts linked by joints naturally form a graph structure. By analyzing the graph, we are able to recognize the human. Moreover, graph is a generalized data structure that grid and sequence can be viewed as a special case of graph. Viewing an image as a graph is more flexible and effective for visual perception.

Based on the graph representation of images, we build the vision graph neural network (ViG for short) for visual tasks. Instead of treating each pixel as a node which will result in too many nodes (>10K), we divide the input image to a number of patches and view each patch as a node. After constructing the graph of image patches, we use our ViG model to transform and exchange information among all the nodes. The basic cells of ViG include two parts: Grapher and FFN (feed-forward network) modules. Grapher module is constructed based on graph convolution for graph information processing. To alleviate over-smoothing phenomenon of conventional GNN, a FFN module is utilized for node feature transformation and encouraging node diversity. With Grapher and FFN modules, we build our ViG models in both isotropic and pyramid manners. In the experiments, we demonstrate the effectiveness of ViG model on visual tasks like image classification and object detection. For instance, our Pyramid ViG-S achieves 82.1\% top-1 accuracy on ImageNet classification task, which outperforms the representative CNN (ResNet~\cite{resnet}), MLP (CycleMLP~\cite{cyclemlp}) and transformer (Swin-T~\cite{swin}) with similar FLOPs (about 4.5G). To the best of our knowledge, our work is the first to successfully apply graph neural network on large-scale visual tasks. We hope our work will inspire the community to further explore more powerful network architectures.

\begin{figure}[tp]
	\centering
	\small
	\setlength{\tabcolsep}{10pt}{
		\begin{tabular}{ccc}
			\makecell*[c]{\includegraphics[width=0.2\linewidth]{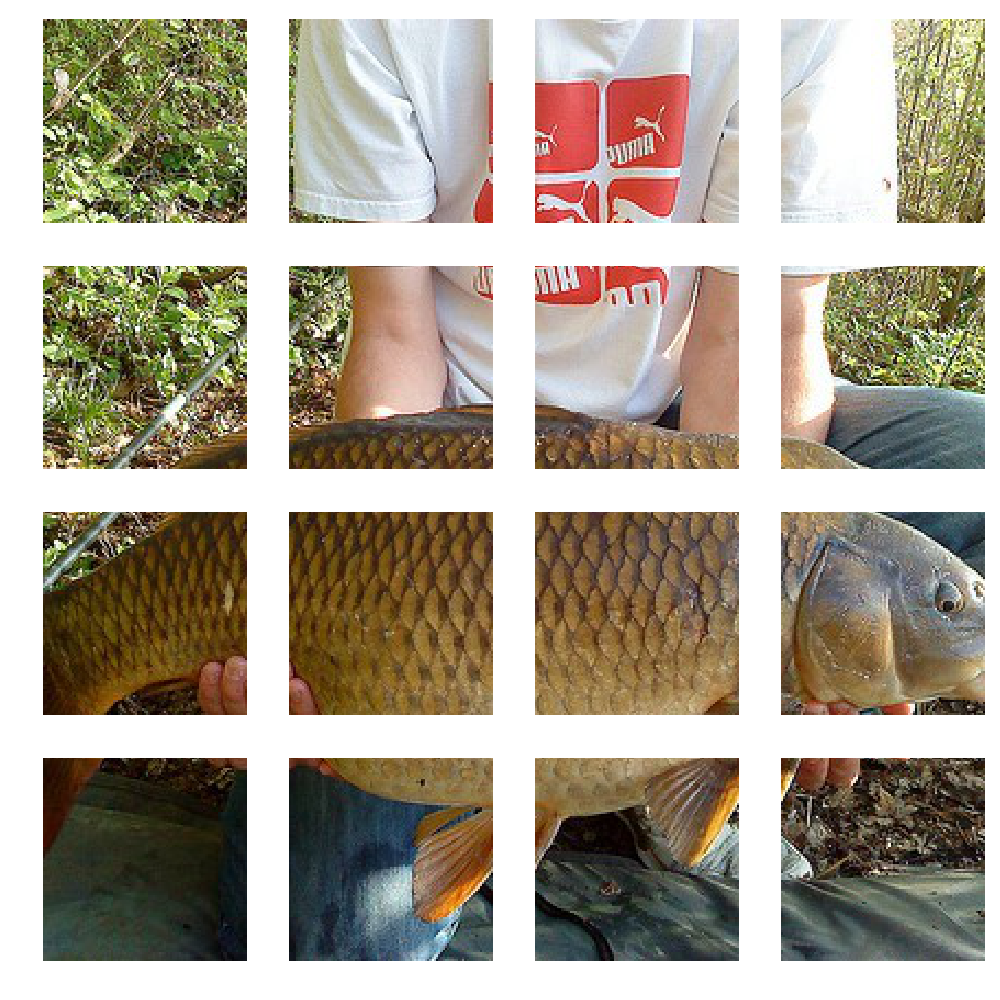}}  &
			\makecell*[c]{\includegraphics[width=0.3\linewidth]{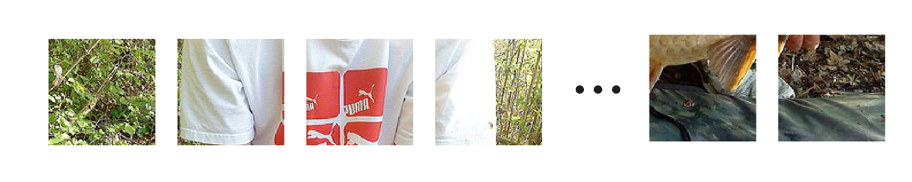}}  & \makecell*[c]{\includegraphics[width=0.35\linewidth]{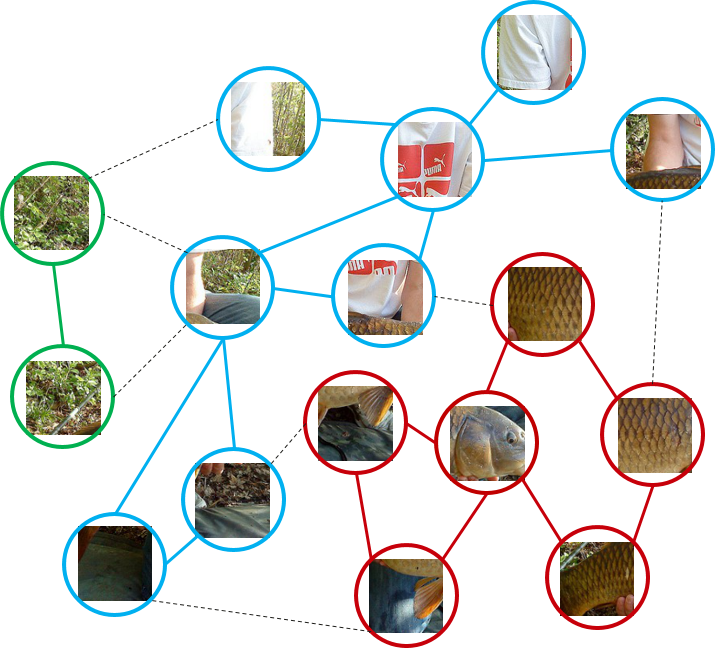}}
			\\
			(a) Grid structure. & (b) Sequence structure. & (b) Graph structure.
		\end{tabular}
	}
	\caption{Illustration of the grid, sequence and graph representation of the image. In the grid structure, the pixels or patches are ordered only by the spatial position. In the sequence structure, the 2D image is transformed to a sequence of patches. In the graph structure, the nodes are linked by its content and are not constrained by the local position.}
	\label{Fig:image}	
	\vspace{-1.0em}
\end{figure}

\section{Related Work}
\label{sec:related}

In this section, we first revisit the backbone networks in computer vision. Then we review the development of graph neural network, especially GCN and its applications on visual tasks.

\subsection{CNN, Transformer and MLP for Vision}
The mainstream network architecture in computer vision used to be convolutional network~\cite{lecun1998gradient,alexnet,resnet}. Starting from LeNet~\cite{lecun1998gradient}, CNNs have been successfully used in various visual tasks, \eg, image classification~\cite{alexnet}, object detection~\cite{fasterRCNN} and semantic segmentation~\cite{fcn}. The CNN architecture is evolving rapidly in the last ten years. The representative works include ResNet~\cite{resnet}, MobileNet~\cite{mobilenet} and NAS~\cite{nas1,yang2020cars}.
Vision transformer was introduced for visual tasks from 2020~\cite{vit-survey,vit,detr,ipt}. From then on, a number of variants of ViT~\cite{vit} were proposed to improve the performance on visual tasks. The main improvements include pyramid architecture~\cite{pvt,swin}, local attention~\cite{tnt,swin} and position encoding~\cite{wu2021rethinking}.
Inspired by vision transformer, MLP is also explored in computer vision~\cite{mixer,resmlp}. With specially designed modules~\cite{cyclemlp,asmlp,hire,wavemlp}, MLP can achieve competitive performance and work on general visual tasks like object detection and segmentation.

\subsection{Graph Neural Network}
The earliest graph neural network was initially outlined in~\cite{gori2005new,scarselli2008graph}. Micheli~\cite{micheli2009neural} proposed the early form of spatial-based graph convolutional network by architecturally composite nonrecursive layers. In recent several years, the variants of spatial-based GCNs have been introduced, such as~\cite{niepert2016learning,atwood2016diffusion,gilmer2017neural}. Spectral-based GCN was first presented by Bruna~\etal~\cite{bruna2013spectral} that introduced graph convolution based on the spectral graph theory. Since this time, a number of works to improve and extend spectral-based GCN have been proposed~\cite{henaff2015deep,defferrard2016convolutional,kipf2016semi}. The GCNs are usually applied on graph data, such as social networks~\cite{sage}, citation
networks~\cite{sen2008collective} and biochemical graphs~\cite{wale2008comparison}.

The applications of GCN in the field of computer vision~\cite{xu2017scene,landrieu2018large,wang2019learning,jing2022learning} mainly include point clouds classification, scene graph generation, and action recognition. A point cloud is a set of 3D points in space which is usually collected by LiDAR scans. GCN has been explored for classifying and segmenting points clouds~\cite{landrieu2018large,edgeconv,yang2020distilling}.
Scene graph generation aims to parse the input image intro a graph with the objects and their relationship, which is usually solved by combining object detector and GCN~\cite{xu2017scene,yang2018graph}.
By processing the naturally formed graph of linked human joints, GCN was utilized on human action recognition task~\cite{jain2016structural,yan2018spatial}.
GCN can only tackle specific visual tasks with naturally constructed graph. For general applications in computer vision, we need a GCN-based backbone network that directly processes the image data.

\section{Approach}
\label{sec:approach}
In this section, we describe how to transform an image to a graph and introduce vision GNN architectures to learn visual representation.

\begin{figure}[htp] 
	\centering
	\includegraphics[width=1.0\linewidth]{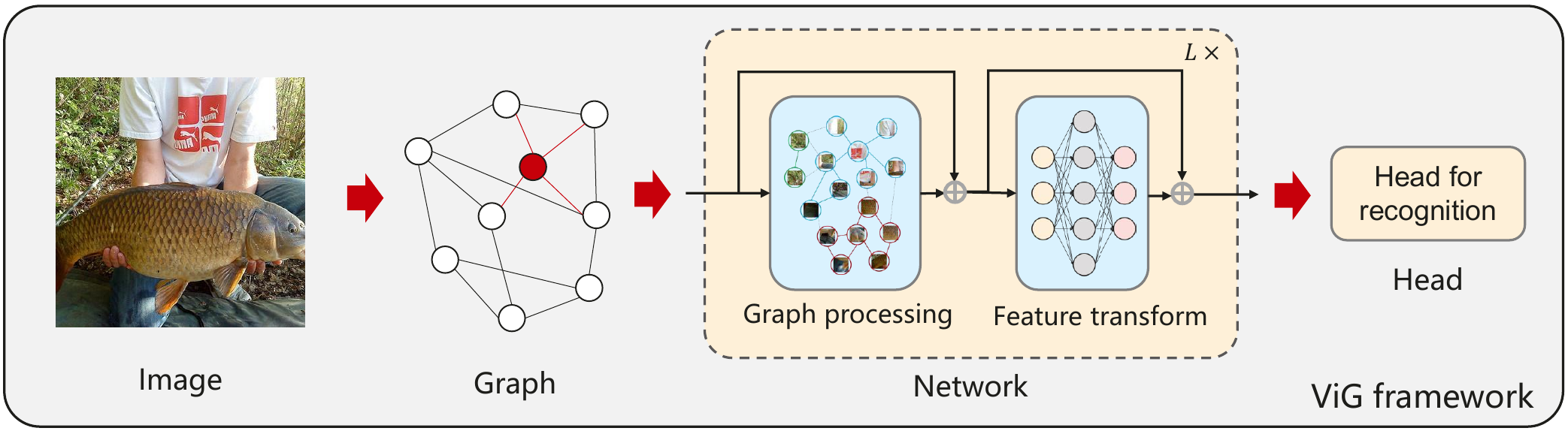}	
	\vspace{-1em}
	\caption{The framework of the proposed ViG model.}
	\label{Fig:vig}
	\vspace{-0.5em}
\end{figure}

\subsection{ViG Block}
\paragraph{Graph Structure of Image.}
For an image with size of $H\times W\times 3$, we divided it into $N$ patches. By transforming each patch into a feature vector $\mathbf{x}_i\in\mathbb{R}^{D}$, we have $X=[\mathbf{x}_1,\mathbf{x}_2,\cdots,\mathbf{x}_N]$ where $D$ is the feature dimension and $i=1,2,\cdots,N$. These features can be viewed as a set of unordered nodes which are denoted as $\mathcal{V}=\{v_1,v_2,\cdots,v_N\}$. For each node $v_i$, we find its $K$ nearest neighbors $\mathcal{N}(v_i)$ and add an edge $e_{ji}$ directed from $v_j$ to $v_i$ for all $v_j\in\mathcal{N}(v_i)$. Then we obtain a graph $\mathcal{G}=(\mathcal{V},\mathcal{E})$ where $\mathcal{E}$ denote all the edges. We denote the graph construction process as $\mathcal{G}=G(X)$ in the following. By viewing the image as a graph data, we explore how to utilize GNN to extract its representation.

The advantages of graph representation of the image include: 1) graph is a generalized data structure that grid and sequence can be viewed as a special case of graph; 2) graph is more flexible than grid or sequence to model the complex object as an object in the image is usually not quadrate whose shape is irregular; 3) an object can be viewed as a composition of parts (\eg, a human can be roughly divided into head, upper body, arms and legs), and graph structure can construct the connections among those parts; 4) the advanced research on GNN can be transferred to address visual tasks.

\paragraph{Graph-level processing.}
To be general, we start from the features $X\in\mathbb{R}^{N\times D}$. We first construct a graph based on the features: $\mathcal{G}=G(X)$. A graph convolutional layer can exchange information between nodes by aggregating features from its neighbor nodes. Specifically, graph convolution operates as follows:
\begin{equation}
	\begin{aligned}
		\mathcal{G}' &= {F}(\mathcal{G}, \mathcal{W})\\
		&= Update(Aggregate(\mathcal{G}, W_{agg}), W_{update}),
	\end{aligned}
\end{equation}
where $W_{agg}$ and $W_{update}$ are the learnable weights of the aggregation and update operations, respectively. More concretely, the aggregation operation computes the representation of a node by aggregating features of neighbor nodes, and the update operation further merge the aggregated feature:
\begin{equation}
	\mathbf{x}'_i = h(\mathbf{x}_i, g(\mathbf{x}_i, \mathcal{N}(\mathbf{x}_i), W_{agg}), W_{update}),
\end{equation}
where $\mathcal{N}(\mathbf{x}_i^{l})$ is the set of neighbor nodes of $\mathbf{x}_i^{l}$. Here we adopt max-relative graph convolution~\cite{deepgcn} for its simplicity and efficiency:
\begin{align}
	g(\cdot) &= \mathbf{x}''_i = [\mathbf{x}_i,\max(\{\mathbf{x}_j-\mathbf{x}_i|j\in\mathcal{N}(\mathbf{x}_i)\}],\\
	h(\cdot) &= \mathbf{x}'_i = \mathbf{x}''_iW_{update},
\end{align}
where the bias term is omitted. The above graph-level processing can be denoted as $X'=\text{GraphConv}(X)$. 

We further introduce multi-head update operation of graph convolution. The aggregated feature $\mathbf{x}''_i$ is first split into $h$ heads, \ie, $\textit{head}^1, \textit{head}^2, \cdots, \textit{head}^h$ and then these heads are updated with different weights respectively. All the heads can be updated in parallel and are concatenated as the final values:
\begin{equation}\label{eq:multihead}
	\mathbf{x}'_i = [\textit{head}^1W_{update}^1, \textit{head}^2W_{update}^2,\cdots, \textit{head}^hW_{update}^h].
\end{equation}
Multi-head update operation allows the model to update information in multiple representation subspaces, which is beneficial to the feature diversity.

\begin{wrapfigure}{r}{0.4\textwidth}
	\vspace{-1.5em}
	\begin{center}
		\includegraphics[width=0.4\textwidth]{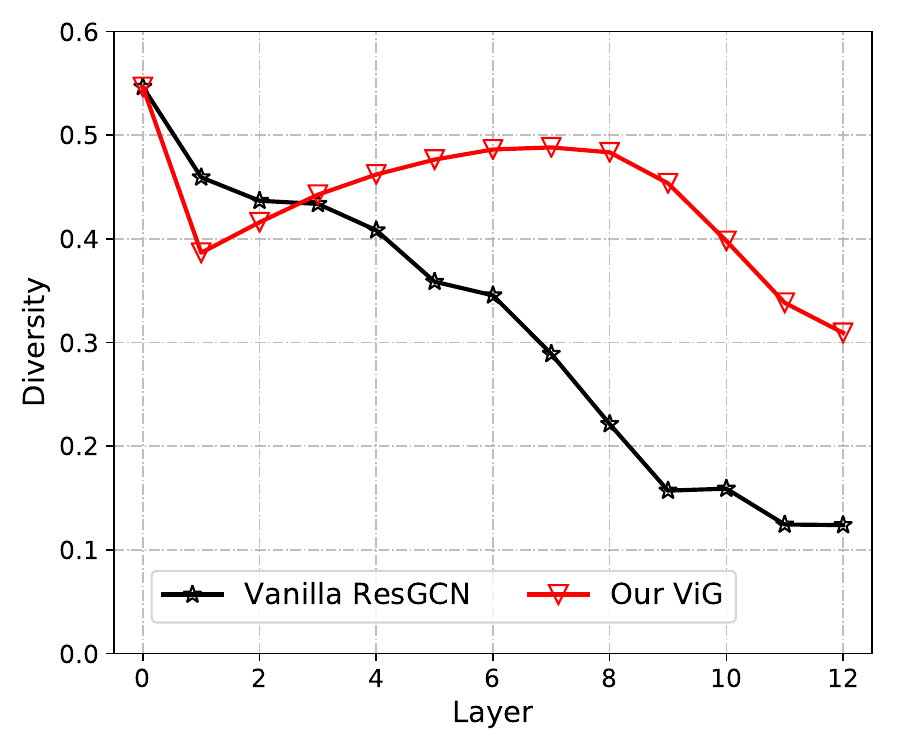}
	\end{center}
	\vspace{-1.em}
	\caption{Feature diversity of nodes as layer changes.}\label{Fig:diversity}
	\vspace{-1.em}
\end{wrapfigure}

\paragraph{ViG block.}
The previous GCNs usually repeatedly use several graph convolution layers to extract aggregated feature of the graph data. The over-smoothing phenomenon in deep GCNs~\cite{li2018deeper,oono2019graph} will decrease the distinctiveness of node features and lead to performance degradation for visual recognition, as shown in Figure~\ref{Fig:diversity} where diversity is measured as $\|X-\mathbf{1}\tilde{\mathbf{x}}^T\|$ with $\tilde{\mathbf{x}}=\arg\min_{\tilde{\mathbf{x}}}\|X-\mathbf{1}\tilde{\mathbf{x}}^T\|$~\cite{dong2021attention}. To alleviate this issue, we introduce more feature transformations and nonlinear activations in our ViG block.

We apply a linear layer before and after the graph convolution to project the node features into the same domain and increase the feature diversity. A nonlinear activation function is inserted after graph convolution to avoid layer collapse. We call the upgraded module as Grapher module. In practice, given the input feature $X\in\mathbb{R}^{N\times D}$, the Grapher module can be expressed as
\begin{equation}\label{eq:gcn}
	Y = \sigma(\text{GraphConv}(XW_{in}))W_{out} + X,
\end{equation}
where $Y\in\mathbb{R}^{N\times D}$, $W_{in}$ and $W_{out}$ are the weights of fully-connected layers, $\sigma$ is the activation function, \eg, ReLU and GeLU~\cite{gelu}, and the bias term is omitted. 

To further encourage the feature transformation capacity and relief the over-smoothing phenomenon, we utilize feed-forward network (FFN) on each node. The FFN module is a simple multi-layer perceptron with two fully-connected layers:
\begin{equation}\label{eq:ffn}
	Z = \sigma(YW_{1})W_{2} + Y,
\end{equation}
where $Z\in\mathbb{R}^{N\times D}$, $W_{1}$ and $W_{2}$ are the weights of fully-connected layers, and the bias term is omitted. The hidden dimension of FFN is usually larger than $D$. In both Grapher and FFN modules, batch normalization is applied after every fully-connected layer or graph convolution layer, which is omitted in Eq.~\ref{eq:gcn} and \ref{eq:ffn} for concision.
A stack of Grapher module and FFN module constitutes the ViG block which serves as the basic building unit for constructing a network. Based on the graph representation of images and the proposed ViG block, we can build the ViG network for visual tasks as shown in Figure~\ref{Fig:vig}. Compared to vanilla ResGCN~\cite{deepgcn}, our ViG can maintain the feature diversity (Figure~\ref{Fig:diversity}) as the layer goes deeper so as to learn discriminative representations.

\subsection{Network Architecture}
In the field of computer vision, the commonly-used transformer usually has an isotropic architecture (\eg, ViT~\cite{vit}), while CNNs prefer to use pyramid architecture (\ie, ResNet~\cite{resnet}). To have a extensive comparison with other types of neural networks, we build two kinds of network architectures for ViG, \ie, isotropic architecture and pyramid architecture.

\paragraph{Isotropic architecture.}
Isotropic architecture means the main body has features with equal size and shape throughout the network, such as ViT~\cite{vit} and ResMLP~\cite{resmlp}. We build three versions of isotropic ViG architecture with different models sizes, \ie, ViG-Ti, S and B. The number of nodes is set as $N=196$. To enlarge the receptive field gradually, the number of neighbor nodes $K$ increases from 9 to 18 linearly as the layer goes deep in these three models. The number of heads is set as $h=4$ by default. The details are listed in Table~\ref{tab:vig}.

\begin{table}[htp]
	\vspace{-0.em}
	\small 
	\centering
	\caption{Variants of our isotropic ViG architecture. The FLOPs are calculated for the image with 224$\times$224 resolution. `Ti' denotes tiny, `S' denotes small, and `B' denotes base. }\label{tab:vig}
	\renewcommand{\arraystretch}{1.05}
	\setlength{\tabcolsep}{8pt}{
		\begin{tabular}{l|c|c|c|c}
			\toprule[1.5pt]
			{Model}  & {Depth} & Dimension $D$ & Params (M) & FLOPs (B) \\
			\midrule
			ViG-Ti & 12 & 192 & 7.1 & 1.3 \\
			ViG-S & 16 & 320 & 22.7 & 4.5 \\
			ViG-B & 16 & 640 & 86.8 & 17.7 \\
			\bottomrule[1pt]
		\end{tabular}
	}
	\vspace{-1.0em}
\end{table}

\paragraph{Pyramid architecture.}
Pyramid architecture considers the multi-scale property of images by extracting features with gradually smaller spatial size as the layer goes deeper, such as ResNet~\cite{resnet} and PVT~\cite{pvt}. Empirical evidences show that pyramid architecture is effective for visual tasks~\cite{pvt}. Thus, we utilize the advanced design and build four versions of pyramid ViG models. The details are shown in Table~\ref{tab:pvig}. Note that we utilize the spatial reduction~\cite{pvt} in the first two stages to handle large number of nodes.

\begin{table*}[htp]
	\vspace{-0.5em}
	\small
	\centering
	\caption{Detailed settings of Pyramid ViG series. $D$: feature dimension, $E$: hidden dimension ratio in FFN, $K$: number of neighbors in GCN, $H\times W$: input image size. `Ti' denotes tiny, `S' denotes small, `M' denotes medium, and `B' denotes base. 
	}
	\renewcommand\arraystretch{1.1}
	\setlength{\tabcolsep}{6pt}
	\begin{tabular}{c|c|c|c|c|c}
	\Xhline{1.2pt}
	Stage & Output size & {PyramidViG-Ti} & {PyramidViG-S} & {PyramidViG-M} & {PyramidViG-B}
	\\
	\hline
	Stem & $\frac{H}{4}\times\frac{W}{4}$ & {Conv$\times3$} & {Conv$\times3$} & Conv$\times3$ & {Conv$\times3$}
	\\
	\hline
	{Stage 1} & {$\frac{H}{4}\times\frac{W}{4}$} 
	& $\begin{bmatrix}D=48\\E=4\\K=9\end{bmatrix}$$\times$2 
	& $\begin{bmatrix}D=80\\E=4\\K=9\end{bmatrix}$$\times$2 
	& $\begin{bmatrix}D=96\\E=4\\K=9\end{bmatrix}$$\times$2 
	& $\begin{bmatrix}D=128\\E=4\\K=9\end{bmatrix}$$\times$2 
	\\
	\hline
	Downsample & $\frac{H}{8}\times\frac{W}{8}$ & Conv & Conv & Conv & Conv
	\\
	\hline
	{Stage 2} & {$\frac{H}{8}\times\frac{W}{8}$} 
	& $\begin{bmatrix}D=96\\E=4\\K=9\end{bmatrix}$$\times$2 
	& $\begin{bmatrix}D=160\\E=4\\K=9\end{bmatrix}$$\times$2 
	& $\begin{bmatrix}D=192\\E=4\\K=9\end{bmatrix}$$\times$2 
	& $\begin{bmatrix}D=256\\E=4\\K=9\end{bmatrix}$$\times$2 
	\\
	\hline
	Downsample & $\frac{H}{16}\times\frac{W}{16}$ & Conv & Conv & Conv & Conv
	\\
	\hline
	{Stage 3} & {$\frac{H}{16}\times\frac{W}{16}$} 
	& $\begin{bmatrix}D=240\\E=4\\K=9\end{bmatrix}$$\times$6 
	& $\begin{bmatrix}D=400\\E=4\\K=9\end{bmatrix}$$\times$6 
	& $\begin{bmatrix}D=384\\E=4\\K=9\end{bmatrix}$$\times$16 
	& $\begin{bmatrix}D=512\\E=4\\K=9\end{bmatrix}$$\times$18 
	\\
	\hline
	Downsample & $\frac{H}{32}\times\frac{W}{32}$ & Conv & Conv & Conv & Conv
	\\
	\hline
	{Stage 4} & {$\frac{H}{32}\times\frac{W}{32}$} 
	& $\begin{bmatrix}D=384\\E=4\\K=9\end{bmatrix}$$\times$2 
	& $\begin{bmatrix}D=640\\E=4\\K=9\end{bmatrix}$$\times$2 
	& $\begin{bmatrix}D=768\\E=4\\K=9\end{bmatrix}$$\times$2 
	& $\begin{bmatrix}D=1024\\E=4\\K=9\end{bmatrix}$$\times$2 
	\\
	\hline
	Head & $1\times1$ & {Pooling \& MLP} & {Pooling \& MLP} & {Pooling \& MLP} & {Pooling \& MLP}\\
	\hline
	\multicolumn{2}{c|}{Parameters (M)} & {10.7} & {27.3} & {51.7} & {92.6}
	\\
	\hline
	\multicolumn{2}{c|}{FLOPs (B)} & {1.7} & {4.6} & {8.9} & {16.8}
	\\
	\Xhline{1.2pt}
\end{tabular}
	\label{tab:pvig}
	\vspace{-1.0em}
\end{table*}

\paragraph{Positional encoding.}
In order to represent the position information of the nodes, we add a positional encoding vector to each node feature:
\begin{equation}\label{eq:abs-pos}
	\mathbf{x}_i\leftarrow \mathbf{x}_i+\mathbf{e}_i,
\end{equation}
where $\mathbf{e}_i\in\mathbb{R}^{D}$. The absolute positional encoding as described in Eq.~\ref{eq:abs-pos} is applied in both iostropic and pyramid architectures. For pyramid ViG, we further include relative positional encoding by following the advanced designs like Swin Transformer~\cite{swin}. For node $i$ and $j$, the relative positional distance between them is $\mathbf{e}_i^T\mathbf{e}_j$, which will be added into the feature distance for constructing the graph.

\section{Experiments}
\label{sec:exp}
In this section, we conduct experiments to demonstrate the effectiveness of ViG models on visual tasks including image recognition and object detection.

\subsection{Datasets and Experimental Settings}
\label{sec:dataset}
\paragraph{Datasets.}
In image classification task, the widely-used benchmark ImageNet ILSVRC 2012~\cite{imagenet} is used in the following experiments. ImageNet has 1.2M training images and 50K validation images, which belong to 1000 categories. For the license of ImageNet dataset, please refer to \url{http://www.image-net.org/download}.
For object detection, we use COCO 2017~\cite{coco} dataset with 80 object categories. COCO 2017 contains 118K training images and 5K validation images. For the licenses of these datasets, please refer to \url{https://cocodataset.org/#home}.


\begin{wraptable}{r}{0.6\textwidth}
	\vspace{-1.5em}
	\centering
	\small
	\caption{Training hyper-parameters for ImageNet.}
	\label{table-hyper}
	\setlength{\tabcolsep}{8pt}
	\begin{tabular}{l|cccc}
		\Xhline{1.2pt}
		(Pyramid) ViG & Ti & S & M & B \\ 
		\hline
		Epochs & \multicolumn{4}{c}{300} \\
		Optimizer & \multicolumn{4}{c}{AdamW~\cite{adamw}} \\
		Batch size & \multicolumn{4}{c}{1024} \\
		Start learning rate (LR) & \multicolumn{4}{c}{2e-3} \\
		Learning rate schedule & \multicolumn{4}{c}{Cosine} \\
		Warmup epochs & \multicolumn{4}{c}{20} \\
		Weight decay & \multicolumn{4}{c}{0.05} \\
		Label smoothing~\cite{label-smooth} & \multicolumn{4}{c}{0.1} \\
		Stochastic path~\cite{huang2016deep} & 0.1 & 0.1 & 0.1 & 0.3 \\
		Repeated augment~\cite{hoffer2020augment} & \multicolumn{4}{c}{$\checkmark$} \\
		RandAugment~\cite{randaugment} & \multicolumn{4}{c}{$\checkmark$} \\
		Mixup prob.~\cite{mixup} & \multicolumn{4}{c}{0.8} \\
		Cutmix prob.~\cite{cutmix} & \multicolumn{4}{c}{1.0} \\
		Random erasing prob.~\cite{erasing} & \multicolumn{4}{c}{0.25} \\
		Exponential moving average & \multicolumn{4}{c}{0.99996} \\
		\Xhline{1.2pt}
	\end{tabular}
	\vspace{-0.5em}
\end{wraptable}

\paragraph{Experimental Settings.}
For all the ViG models, we utilize dilated aggregation~\cite{deepgcn} in Grapher module and set the dilated rate as $\lceil l/4\rceil$ for the $l$-th layer. GELU~\cite{gelu} is used as the nonlinear activation function in Eq.~\ref{eq:gcn} and~\ref{eq:ffn}.
For ImageNet classification, we use the commonly-used training strategy proposed in DeiT~\cite{deit} for fair comparison. The data augmentation includes RandAugment~\cite{randaugment}, Mixup~\cite{mixup}, Cutmix~\cite{cutmix}, random erasing~\cite{erasing} and repeated augment~\cite{hoffer2020augment}. The details are shown in Table~\ref{table-hyper}. For COCO detection task, we take RetinaNet~\cite{retinanet} and Mask R-CNN~\cite{maskrcnn} as the detection frameworks and use our Pyramid ViG as backbone. All the models are trained on COCO 2017 training set in ``1$\times$'' schedule and evaluated on validation set. We implement the networks using PyTroch~\cite{pytorch} and MindSpore~\cite{mindspore} and train all our models on 8 NVIDIA V100 GPUs.

\begin{table}[htp]
	\vspace{-0.5em}
	\small 
	\centering
	\caption{Results of ViG and other isotropic networks on ImageNet. {\color{blue}$\spadesuit$} CNN, {\color{orange}$\blacksquare$} MLP, {\color{green}$\blacklozenge$} Transformer, {\color{red}$\bigstar$} GNN.}\label{tab:iso}
	\vspace{-0.em}
	\renewcommand{\arraystretch}{1.0}
	\setlength{\tabcolsep}{10pt}{
		\begin{tabular}{l|c|c|c|c|c}
			\toprule[1.5pt]
			Model    &  Resolution  & Params (M) & FLOPs (B) & Top-1 & Top-5 \\
			\midrule
			{\color{blue}$\spadesuit$} ResMLP-S12 conv3x3~\cite{resmlp} & 	224$\times$224 & 16.7 & 3.2 & 77.0 & -\\
			{\color{blue}$\spadesuit$} ConvMixer-768/32~\cite{convmixer} & 224$\times$224 & 21.1 & 20.9 & 80.2 & - \\
			{\color{blue}$\spadesuit$} ConvMixer-1536/20~\cite{convmixer} & 224$\times$224 & 51.6 & 51.4 & 81.4 & - \\
			\midrule
			{\color{green}$\blacklozenge$} ViT-B/16~\cite{vit}  & 384$\times$384 & 86.4 & 55.5 &  77.9 & - \\
			{\color{green}$\blacklozenge$} DeiT-Ti~\cite{deit}  & 224$\times$224 & 5.7 & 1.3 & 72.2 & 91.1 \\			
			{\color{green}$\blacklozenge$} DeiT-S~\cite{deit}  & 224$\times$224 & 22.1 & 4.6 & 79.8 & 95.0 \\
			{\color{green}$\blacklozenge$} DeiT-B~\cite{deit}   & 224$\times$224 & 86.4 & 17.6 & 81.8 & 95.7 \\
			\midrule			
			{\color{orange}$\blacksquare$} ResMLP-S24~\cite{resmlp} & 224$\times$224 & 30 & 6.0 & 79.4 & 94.5 \\
			{\color{orange}$\blacksquare$} ResMLP-B24~\cite{resmlp}   & 224$\times$224 & 116 & 23.0 & 81.0 & 95.0 \\
			{\color{orange}$\blacksquare$} Mixer-B/16~\cite{mixer}   & 224$\times$224 & 59 & 11.7 & 76.4 & - \\
			\midrule
			{\color{red}$\bigstar$} ViG-Ti (ours) & 224$\times$224 & 7.1 & 1.3 & \textbf{73.9} & \textbf{92.0} \\			
			{\color{red}$\bigstar$} ViG-S (ours) & 224$\times$224 & 22.7 & 4.5 & \textbf{80.4} & \textbf{95.2} \\
			{\color{red}$\bigstar$} ViG-B (ours) & 224$\times$224 & 86.8 & 17.7 & \textbf{82.3} & \textbf{95.9} \\
			\bottomrule[1pt]
		\end{tabular}
	}
	\vspace{-0.em}
\end{table}

\subsection{Main Results on ImageNet}

\paragraph{Isotropic ViG}
The neural network with iostropic architecture keeps the feature size unchanged in its main computational body, which is easy to scale and is friendly for hardware acceleration. This scheme is widely used in transformer models for natural language processing~\cite{Att}. The recent neural networks in vision also explore it such as ConvMixer~\cite{mixer}, ViT~\cite{vit} and ResMLP~\cite{resmlp}. We compare our isotropic ViG with the existing iostropic CNNs~\cite{resmlp,mixer}, transformers~\cite{vit,deit} and MLPs~\cite{resmlp,mixer} in Table~\ref{tab:iso}. From the results, ViG performs better than other types of networks. For example, our ViG-Ti achieves 73.9\% top-1 accuracy which is 1.7\% higher than DeiT-Ti model with similar computational cost.

\paragraph{Pyramid ViG}
The pyramid architecture gradually shrinks the spatial size of feature maps as the network deepens, which can leverage the scale-invariant property of images and produce multi-scale features. The advanced networks usually adopt the pyramid architecture, such as ResNet~\cite{resnet}, Swin Transformer~\cite{swin} and CycleMLP~\cite{cyclemlp}. We compare our Pyramid ViG with those representative pyramid networks in Table~\ref{tab:pvig-sota}. Our Pyramid ViG series can outperform or be comparable to the state-of-the-art pyramid networks including CNN, MLP and transformer. This indicates that graph neural network can work well on visual tasks and has the potential to be a basic component in computer vision system.

\begin{table}[htp]
	\vspace{-0.em}
	\small 
	\centering
	\caption{Results of Pyramid ViG and other pyramid networks on ImageNet. {\color{blue}$\spadesuit$} CNN, {\color{orange}$\blacksquare$} MLP, {\color{green}$\blacklozenge$} Transformer, {\color{red}$\bigstar$} GNN.}\label{tab:pvig-sota}
	\vspace{-0.em}
	\renewcommand{\arraystretch}{1.0}
	\setlength{\tabcolsep}{10.5pt}{
		\begin{tabular}{l|c|c|c|c|c}
			\toprule[1.5pt]
			Model    &  Resolution  & Params (M) & FLOPs (B) & Top-1 & Top-5 \\
			\midrule
			{\color{blue}$\spadesuit$} ResNet-18~\cite{resnet,wightman2021resnet} & 224$\times$224 & 12 & 1.8 & 70.6 & 89.7 \\
			{\color{blue}$\spadesuit$} ResNet-50~\cite{resnet,wightman2021resnet} & 224$\times$224 & 25.6 & 4.1 & 79.8 & 95.0 \\	
			{\color{blue}$\spadesuit$} ResNet-152~\cite{resnet,wightman2021resnet} & 224$\times$224 & 60.2 & 11.5 & 81.8 & 95.9 \\
			{\color{blue}$\spadesuit$} BoTNet-T3~\cite{botnet} & 224$\times$224 & 33.5 & 7.3 & 81.7 & - \\
			{\color{blue}$\spadesuit$} BoTNet-T3~\cite{botnet} & 224$\times$224 & 54.7 & 10.9 & 82.8 & - \\
			{\color{blue}$\spadesuit$} BoTNet-T3~\cite{botnet} & 256$\times$256 & 75.1 & 19.3 & 83.5 & - \\
			\midrule
			{\color{green}$\blacklozenge$} PVT-Tiny~\cite{pvt}  & 224$\times$224 & 13.2 & 1.9 & 75.1 & - \\			
			{\color{green}$\blacklozenge$} PVT-Small~\cite{pvt}  & 224$\times$224 & 24.5 & 3.8 & 79.8 & - \\
			{\color{green}$\blacklozenge$} PVT-Medium~\cite{pvt}   & 224$\times$224 & 44.2 & 6.7 & 81.2 & - \\
			{\color{green}$\blacklozenge$} PVT-Large~\cite{pvt}  & 224$\times$224 & 61.4 & 9.8 &  81.7 & - \\
			{\color{green}$\blacklozenge$} CvT-13~\cite{cvt}  & 224$\times$224 & 20 & 4.5 & 81.6 & - \\
			{\color{green}$\blacklozenge$} CvT-21~\cite{cvt}  & 224$\times$224 & 32 & 7.1 & 82.5 & - \\
			{\color{green}$\blacklozenge$} CvT-21~\cite{cvt}  & 384$\times$384 & 32 & 24.9 & 83.3 & - \\
			{\color{green}$\blacklozenge$} Swin-T~\cite{swin}  & 224$\times$224 & 29 & 4.5 & 81.3 & 95.5 \\
			{\color{green}$\blacklozenge$} Swin-S~\cite{swin}  & 224$\times$224 & 50 & 8.7 & 83.0 & 96.2 \\
			{\color{green}$\blacklozenge$} Swin-B~\cite{swin}  & 224$\times$224 & 88 & 15.4 & 83.5 & 96.5 \\
			\midrule
			{\color{orange}$\blacksquare$} CycleMLP-B2~\cite{cyclemlp} & 224$\times$224 & 27 & 3.9  & 81.6 & - \\
			{\color{orange}$\blacksquare$} CycleMLP-B3~\cite{cyclemlp} & 224$\times$224 & 38 & 6.9  &  82.4 & - \\
			{\color{orange}$\blacksquare$} CycleMLP-B4~\cite{cyclemlp} & 224$\times$224 & 52 & 10.1  &  83.0 & - \\
			{\color{orange}$\blacksquare$} Poolformer-S12~\cite{metaformer}  & 224$\times$224 & 12 & 2.0 & 77.2 & 93.5 \\
			{\color{orange}$\blacksquare$} Poolformer-S36~\cite{metaformer}  & 224$\times$224 & 31 & 5.2 & 81.4 & 95.5 \\
			{\color{orange}$\blacksquare$} Poolformer-M48~\cite{metaformer}  & 224$\times$224 & 73 & 11.9 & 82.5 & 96.0 \\
			\midrule
			{\color{red}$\bigstar$} Pyramid ViG-Ti (ours) & 224$\times$224 & 10.7 & 1.7 & \textbf{78.2} & \textbf{94.2} \\			
			{\color{red}$\bigstar$} Pyramid ViG-S (ours) & 224$\times$224 & 27.3 & 4.6 & \textbf{82.1} & \textbf{96.0} \\
			{\color{red}$\bigstar$} Pyramid ViG-M (ours) & 224$\times$224 & 51.7 & 8.9 & \textbf{83.1} & \textbf{96.4} \\
			{\color{red}$\bigstar$} Pyramid ViG-B (ours) & 224$\times$224 & 92.6 & 16.8 & \textbf{83.7} & \textbf{96.5} \\
			\bottomrule[1pt]
		\end{tabular}
	}
	\vspace{-0.5em}
\end{table}

\subsection{Ablation Study}
We conduct ablation study of the proposed method on ImageNet classification task and use the isotropic ViG-Ti as the base architecture.

\begin{table}[htp]
	\vspace{-0.em}
	\small 
	\centering
	\caption{ImageNet results of different types of graph convolution. The basic architecture is ViG-Ti.}\label{tab:type}
	\vspace{-0.5em}
	\renewcommand{\arraystretch}{1.0}
	\setlength{\tabcolsep}{10pt}{
		\begin{tabular}{l|c|c|c}
			\toprule[1.5pt]
			GraphConv    &  Params (M) & FLOPs (B) & Top-1 \\
			\midrule
			EdgeConv~\cite{edgeconv} & 7.2 & 2.4 & 74.3 \\
			GIN~\cite{xu2018powerful} & 7.0 & 1.3 & 72.8 \\
			GraphSAGE~\cite{sage} & 7.3 & 1.6 & 74.0 \\
			Max-Relative GraphConv~\cite{deepgcn} & 7.1 & 1.3 & 73.9 \\
			\bottomrule[1pt]
		\end{tabular}
	}
	\vspace{-0.5em}
\end{table}

\paragraph{Type of graph convolution.}
We test the representative variants of graph convolution, including EdgeConv~\cite{edgeconv}, GIN~\cite{xu2018powerful}, GraphSAGE~\cite{sage} and Max-Relative GraphConv~\cite{deepgcn}. From table~\ref{tab:type}, we can see that the top-1 accuracies of different graph convolutions are better than that of DeiT-Ti, indicating the flexibility of ViG architecture. Among them, Max-Relative achieves the best trade-off between FLOPs and accuracy. In rest of the experiments, we use Max-Relative GraphConv by default unless specially stated.

\paragraph{The effects of modules in ViG.}
To make graph neural network adaptive to visual task, we introduce FC layers in Grapher module and utilize FFN block for feature transformation. We evaluate the effects of these modules by ablation study. We change the feature dimension of the compared models to make their FLOPs similar, so as to have a fair comparison. From Table~\ref{tab:ablation}, we can see that directly utilizing graph convolution for image classification performs poorly. Adding more feature transformation by introducing FC and FFN consistently increase the accuracy.

\begin{table}[htp]
	\vspace{-0.em}
	\small 
	\centering
	\caption{The effects of modules in ViG on ImageNet.}\label{tab:ablation}
	\vspace{-0.5em}
	\renewcommand{\arraystretch}{1.0}
	\setlength{\tabcolsep}{10pt}{
		\begin{tabular}{c|c|c|c|c|c}
			\toprule[1.5pt]
			GraphConv &  FC in Grapher module & FFN module  &  Params (M) & FLOPs (B) & Top-1 \\
			\midrule
			\Checkmark & \XSolidBrush & \XSolidBrush & 5.8 & 1.4 & 67.0 \\
			\Checkmark & \Checkmark & \XSolidBrush & 4.4 & 1.4 & 73.4 \\
			\Checkmark & \XSolidBrush & \Checkmark & 7.7 & 1.3 & 73.6 \\
			\Checkmark & \Checkmark & \Checkmark & 7.1 & 1.3 & 73.9 \\
			\bottomrule[1pt]
		\end{tabular}
	}
	\vspace{-1.em}
\end{table}

\paragraph{The number of neighbors.}
In the process of constructing graph, the number of neighbor nodes $K$ is a hyperparameter controlling the aggregated range. Too few neighbors will degrade information exchange, while too many neighbors will lead to over-smoothing. We tune $K$ from 3 to 20 and show the results in Table~\ref{tab:k}. We can see that the number of neighbor nodes in the range from 9 to 15 can perform well on ImageNet classification task.

\begin{table}[htp]
	\vspace{-0.5em}
	\small 
	\centering
	\caption{Top-1 accuracy \emph{vs.} $K$ on ImageNet.}\label{tab:k}
	\vspace{-0.5em}
	\renewcommand{\arraystretch}{1.0}
	\setlength{\tabcolsep}{10pt}{
		\begin{tabular}{l|c|c|c|c|c|c|c}
			\toprule[1.5pt]
			$K$  &  3 & 6 & 9 & 12 & 15 & 20 & 9 to 18\\
			\midrule
			Top-1 & 72.2 & 73.4 & 73.6 & 73.6 & 73.5 & 73.3 & 73.9 \\
			\bottomrule[1pt]
		\end{tabular}
	}
	\vspace{-0.5em}
\end{table}

\paragraph{The number of heads.}
Multi-head update operation allows Grapher module to process node features in different subspaces. The number of heads $h$ in Eq.~\ref{eq:multihead} controls the transformation diversity in subspaces and the FLOPs. We tune $h$ from 1 to 8 and show the results in Table~\ref{tab:head}. The FLOPs and top-1 accuracy on ImageNet changes slightly for different $h$. We select $h=4$ as default value for the optimal trade-off between FLOPs and accuracy.

\begin{table}[H]
	\vspace{-0.5em}
	\small 
	\centering
	\caption{Top-1 accuracy \emph{vs.} $h$ on ImageNet.}\label{tab:head}
	\vspace{-0.5em}
	\renewcommand{\arraystretch}{1.0}
	\setlength{\tabcolsep}{10pt}{
		\begin{tabular}{l|c|c|c|c|c}
			\toprule[1.5pt]
			$h$  &  1 & 2 & 4 & 6 & 8 \\
			\midrule
			FLOPs / Top-1 & 1.6B / 74.2 & 1.4B / 74.0 & 1.3B / 73.9 & 1.2B / 73.7 & 1.2B / 73.7 \\
			\bottomrule[1pt]
		\end{tabular}
	}
	\vspace{-0.5em}
\end{table}

\subsection{Object Detection}
We apply our ViG model on object detection task to evaluate its generalization. To have a fair comparison, we utilize the ImageNet pretrained Pyramid ViG-S as the backbone of RetinaNet~\cite{retinanet} and Mask R-CNN~\cite{maskrcnn} detection frameworks. The models are trained in the commonly-used ``1x'' schedule and FLOPs is calculated with 1280$\times$800 input size. From the results in Table~\ref{table:coco-1x}, we can see that our Pyramid ViG-S performs better than the representative backbones of different types, including ResNet~\cite{resnet}, CycleMLP~\cite{cyclemlp} and Swin Transformer~\cite{swin} on both RetinaNet and Mask R-CNN. The superior results demonstrate the generalization ability of ViG architecture.

\begin{table}[H]
	\small
	\centering
	\renewcommand{\arraystretch}{1.0}
	\setlength\tabcolsep{7.5pt}
	\caption{Object detection and instance segmentation results on COCO val2017. Our Pyramid ViG is compared with other backbones on RetinaNet and Mask R-CNN frameworks.}
	\vspace{-0.5em}
	\label{table:coco-1x}
	\begin{tabular}{l|cc|ccc|ccc}
		\toprule[1.5pt]
		\multirow{2}{*}{Backbone} & \multicolumn{8}{c}{RetinaNet 1$\times$}  \\
		\cline{2-9}
		& Param & FLOPs & mAP & AP$_{50}$ &AP$_{75}$  &AP$_{\rm S}$ &AP$_{\rm M}$ & AP$_{\rm L}$\\
		\midrule		
		ResNet50~\cite{resnet} & 37.7M & 239.3B & 36.3 & 55.3 & 38.6 & 19.3 & 40.0 & 48.8 \\
		ResNeXt-101-32x4d~\cite{resnext} & 56.4M & 319B & 39.9 & 59.6 & 42.7 & 22.3 & 44.2 & 52.5 \\
		PVT-Small~\cite{pvt} & 34.2M & 226.5B & 40.4 & 61.3 & 44.2 & 25.0 & 42.9 & 55.7 \\
		CycleMLP-B2~\cite{cyclemlp} & 36.5M & 230.9B & 40.6 & 61.4 & 43.2 & 22.9 & 44.4 & 54.5 \\
		Swin-T~\cite{swin} & 38.5M & 244.8B & 41.5 & 62.1 & 44.2 & 25.1 & 44.9 & \textbf{55.5} \\
		Pyramid ViG-S (ours) & 36.2M & 240.0B & \textbf{41.8} & \textbf{63.1} & \textbf{44.7} & \textbf{28.5} & \textbf{45.4} & 53.4 \\
		\midrule
		\multirow{2}{*}{Backbone}  &\multicolumn{8}{c}{Mask R-CNN 1$\times$} \\
		\cline{2-9}
		& Param & FLOPs & AP$^{\rm b}$ & AP$_{50}^{\rm b}$ &AP$_{75}^{\rm b}$  &AP$^{\rm m}$ &AP$_{50}^{\rm m}$ & AP$_{75}^{\rm m}$\\
		\midrule		
		ResNet50~\cite{resnet} & 44.2M & 260.1B & 38.0 & 58.6 & 41.4 & 34.4 & 55.1 & 36.7 \\
		PVT-Small~\cite{pvt} & 44.1M & 245.1B & 40.4 & 62.9 & 43.8 & 37.8 & 60.1 & 40.3 \\
		CycleMLP-B2~\cite{cyclemlp} & 46.5M & 249.5B & 42.1 & 64.0 & 45.7 & 38.9 & 61.2 & 41.8 \\
		PoolFormer-S24~\cite{metaformer} & 41.0M & - & 40.1 & 62.2 & 43.4 & 37.0 & 59.1 & 39.6 \\
		Swin-T~\cite{swin} & 47.8M & 264.0B & 42.2 & 64.6 & \textbf{46.2} & 39.1 & 61.6 & \textbf{42.0} \\
		Pyramid ViG-S (ours) & 45.8M & 258.8B & \textbf{42.6} & \textbf{65.2} & {46.0} & \textbf{39.4} & \textbf{62.4} & {41.6} \\
		\bottomrule[1pt]
	\end{tabular}
	\vspace{-0.5em}
\end{table}

\subsection{Visualization}
To better understand how our ViG model works, we visualize the constructed graph structure in ViG-S. In Figure~\ref{Fig:vis}, we show the graphs of two samples in different depths (the 1st and the 12th blocks). The pentagram is the center node, and the nodes with the same color are its neighbors. Two center nodes are visualized as drawing all the edges will be messy. We can observe that our model can select the content-related nodes as the first order neighbors. In the shallow layer, the neighbor nodes tend to be selected based on low-level and local features, such as color and texture. In the deep layer, the neighbors of the center nodes are more semantic and belong to the same category. Our ViG network can gradually link the nodes by its content and semantic representation and help to better recognize the objects.

\begin{figure}[H]
	\vspace{-0.5em}
	\centering
	\small
	\setlength{\tabcolsep}{6pt}{
		\begin{tabular}{ccc}
			\makecell*[c]{\includegraphics[width=0.2\linewidth]{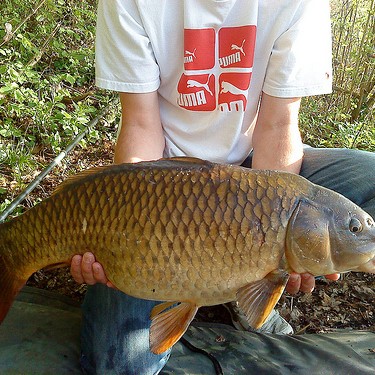}}  &
			\makecell*[c]{\includegraphics[width=0.33\linewidth]{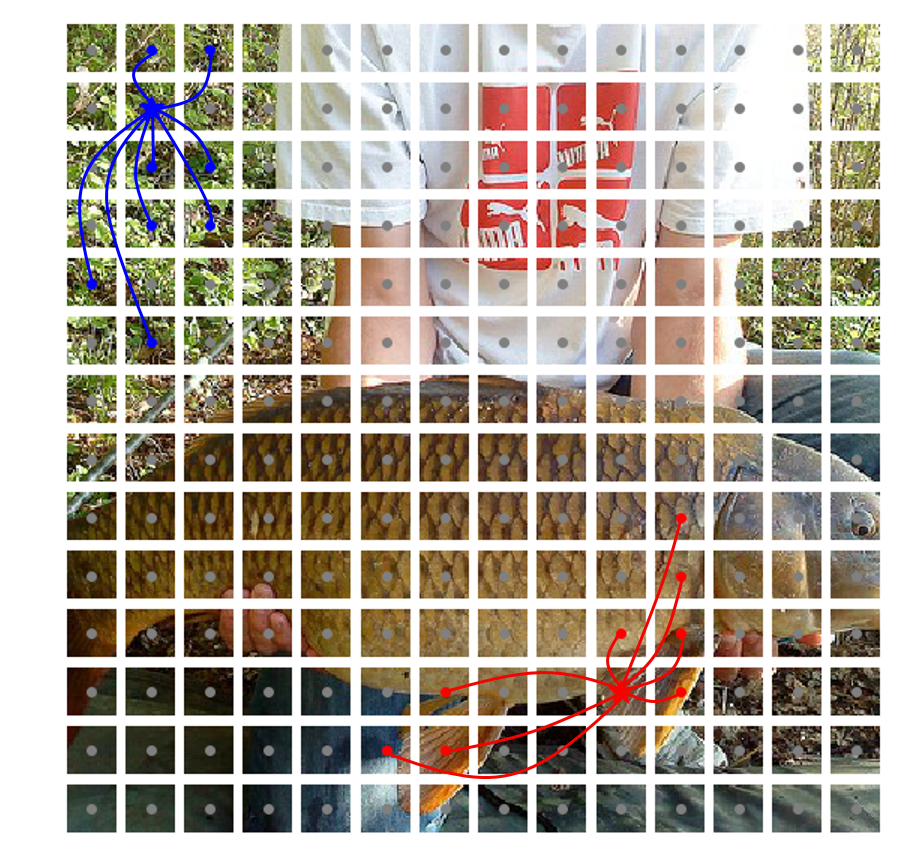}}  & \makecell*[c]{\includegraphics[width=0.33\linewidth]{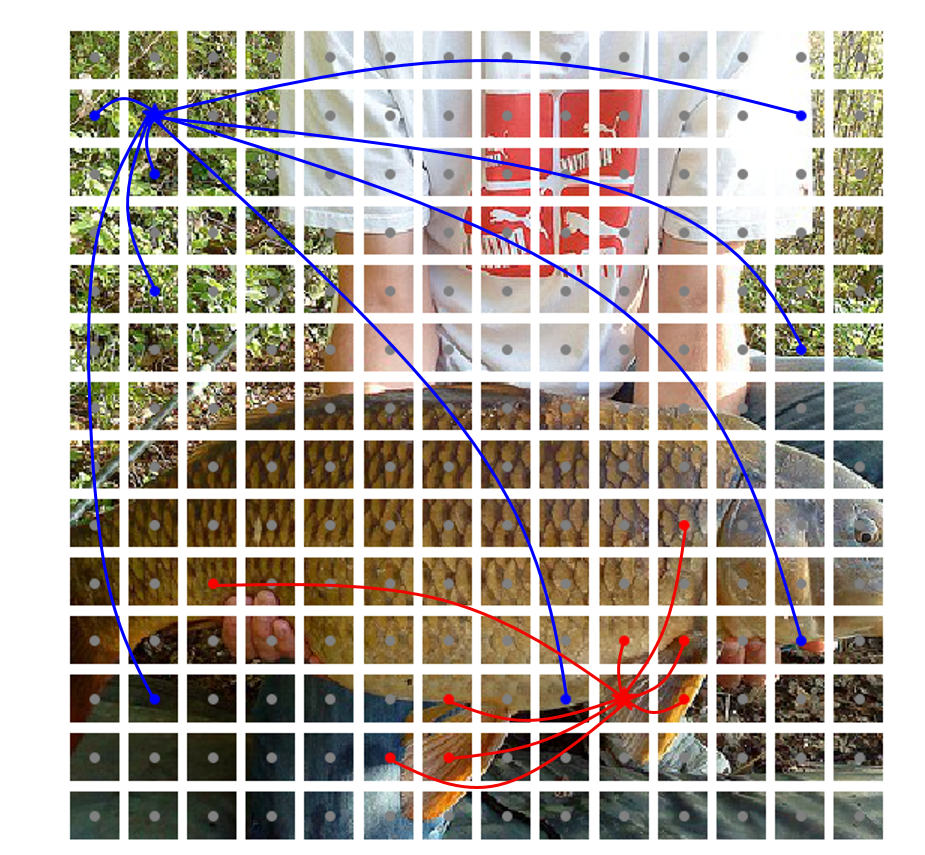}}
			\\
			\makecell*[c]{\includegraphics[width=0.2\linewidth]{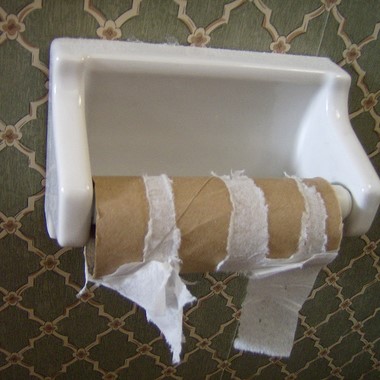}}  &
			\makecell*[c]{\includegraphics[width=0.33\linewidth]{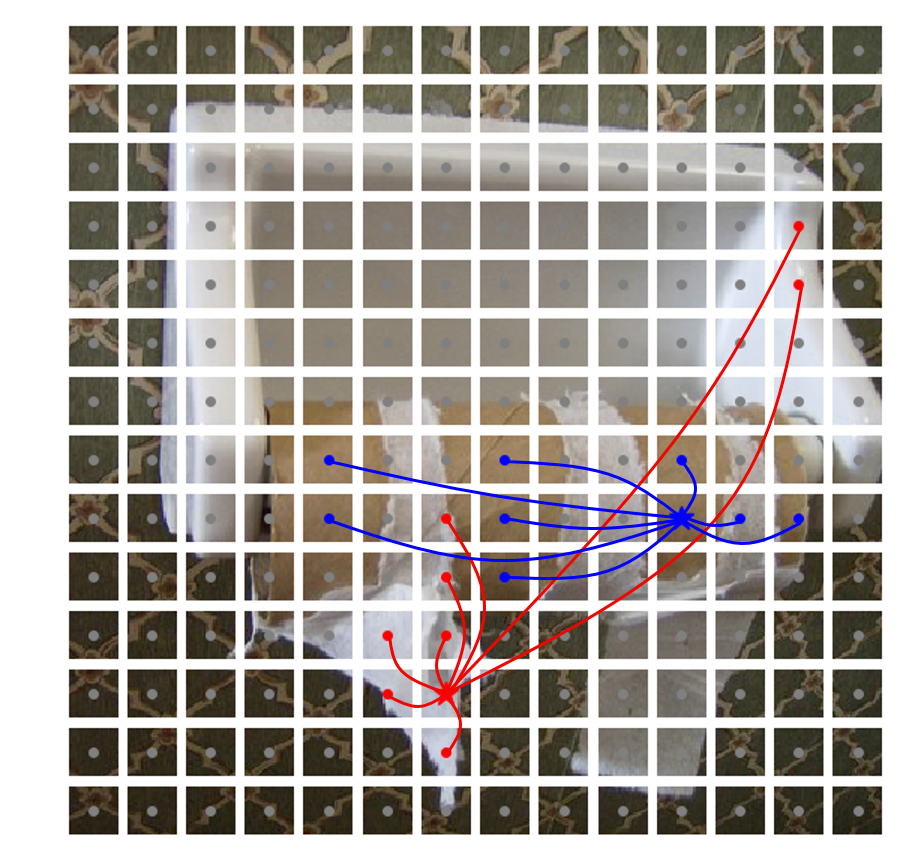}}  & \makecell*[c]{\includegraphics[width=0.33\linewidth]{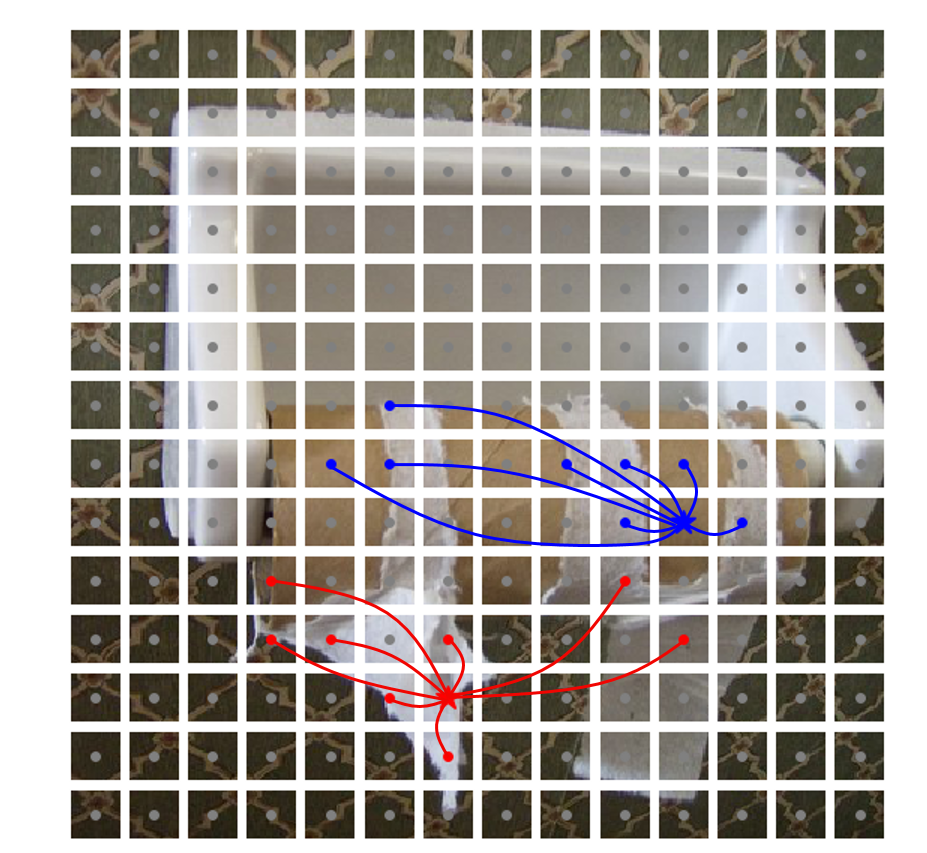}}
			\\
			(a) Input image. & (b) Graph connection in the 1st block. & (c) Graph connection in the 12th block.
		\end{tabular}
	}
	\vspace{-0.5em}
	\caption{Visualization of the constructed graph structure. The pentagram is the center node, and the nodes with the same color are its neighbors in the graph.}
	\label{Fig:vis}
	\vspace{-0em}
\end{figure}

\section{Conclusion}
In this work, we pioneer to study representing the image as graph data and leverage graph neural network for visual tasks. We divide the image into a number of patches and view them as nodes. Constructing graph based on these nodes can better represent the irregular and complex objects in the wild. Directly using graph convolution on the image graph structure has over-smoothing problem and performs poorly. We introduce more feature transformation inside each node to encourage the information diversity. Based on the graph representation of images and improved graph block, we build our vision GNN (ViG) networks with both isotropic and pyramid architectures. Extensive experiments on image recognition and object detection demonstrate the superiority of the proposed ViG architecture. We hope this pioneering work on vision GNN can serve as a basic architecture for general visual tasks.

\section*{Acknowledgement}
This research is supported by NSFC (62072449, 61872345), National Key R\&D Program of China (2021YFB1715800), and Macau Science \&Tech. Fund (0018/2019/AKP). We gratefully acknowledge the support of MindSpore, CANN (Compute Architecture for Neural Networks) and Ascend AI Processor used for this research.

\nocite{xu2019positive,xu2021learning}

{\small
	\bibliographystyle{plain}
	\bibliography{ref}

\begin{thebibliography}{10}

\bibitem{atwood2016diffusion}
James Atwood and Don Towsley.
\newblock Diffusion-convolutional neural networks.
\newblock In {\em NIPS}, pages 2001--2009, 2016.

\bibitem{bruna2013spectral}
Joan Bruna, Wojciech Zaremba, Arthur Szlam, and Yann LeCun.
\newblock Spectral networks and locally connected networks on graphs.
\newblock {\em arXiv preprint arXiv:1312.6203}, 2013.

\bibitem{detr}
Nicolas Carion, Francisco Massa, Gabriel Synnaeve, Nicolas Usunier, Alexander
  Kirillov, and Sergey Zagoruyko.
\newblock End-to-end object detection with transformers.
\newblock In {\em ECCV}, pages 213--229, 2020.

\bibitem{ipt}
Hanting Chen, Yunhe Wang, Tianyu Guo, Chang Xu, Yiping Deng, Zhenhua Liu, Siwei
  Ma, Chunjing Xu, Chao Xu, and Wen Gao.
\newblock Pre-trained image processing transformer.
\newblock In {\em CVPR}, 2021.

\bibitem{cyclemlp}
Shoufa Chen, Enze Xie, Chongjian Ge, Ding Liang, and Ping Luo.
\newblock Cyclemlp: A mlp-like architecture for dense prediction.
\newblock In {\em ICLR}, 2022.

\bibitem{randaugment}
Ekin~D Cubuk, Barret Zoph, Jonathon Shlens, and Quoc~V Le.
\newblock Randaugment: Practical automated data augmentation with a reduced
  search space.
\newblock In {\em CVPR Workshops}, 2020.

\bibitem{defferrard2016convolutional}
Micha{\"e}l Defferrard, Xavier Bresson, and Pierre Vandergheynst.
\newblock Convolutional neural networks on graphs with fast localized spectral
  filtering.
\newblock In {\em NIPS}, volume~29, 2016.

\bibitem{dong2021attention}
Yihe Dong, Jean-Baptiste Cordonnier, and Andreas Loukas.
\newblock Attention is not all you need: Pure attention loses rank doubly
  exponentially with depth.
\newblock In {\em ICML}, pages 2793--2803. PMLR, 2021.

\bibitem{vit}
Alexey Dosovitskiy, Lucas Beyer, Alexander Kolesnikov, Dirk Weissenborn,
  Xiaohua Zhai, Thomas Unterthiner, Mostafa Dehghani, Matthias Minderer, Georg
  Heigold, Sylvain Gelly, et~al.
\newblock An image is worth 16x16 words: Transformers for image recognition at
  scale.
\newblock In {\em ICLR}, 2021.

\bibitem{gilmer2017neural}
Justin Gilmer, Samuel~S Schoenholz, Patrick~F Riley, Oriol Vinyals, and
  George~E Dahl.
\newblock Neural message passing for quantum chemistry.
\newblock In {\em ICML}, pages 1263--1272. PMLR, 2017.

\bibitem{gori2005new}
Marco Gori, Gabriele Monfardini, and Franco Scarselli.
\newblock A new model for learning in graph domains.
\newblock In {\em IJCNN}, volume~2, pages 729--734, 2005.

\bibitem{hire}
Jianyuan Guo, Yehui Tang, Kai Han, Xinghao Chen, Han Wu, Chao Xu, Chang Xu, and
  Yunhe Wang.
\newblock Hire-mlp: Vision mlp via hierarchical rearrangement.
\newblock In {\em CVPR}, 2022.

\bibitem{sage}
William~L Hamilton, Rex Ying, and Jure Leskovec.
\newblock Inductive representation learning on large graphs.
\newblock In {\em NIPS}, pages 1025--1035, 2017.

\bibitem{vit-survey}
Kai Han, Yunhe Wang, Hanting Chen, Xinghao Chen, Jianyuan Guo, Zhenhua Liu,
  Yehui Tang, An~Xiao, Chunjing Xu, Yixing Xu, et~al.
\newblock A survey on vision transformer.
\newblock {\em IEEE Transactions on Pattern Analysis and Machine Intelligence},
  2022.

\bibitem{tnt}
Kai Han, An~Xiao, Enhua Wu, Jianyuan Guo, Chunjing Xu, and Yunhe Wang.
\newblock Transformer in transformer.
\newblock In {\em NeurIPS}, 2021.

\bibitem{maskrcnn}
Kaiming He, Georgia Gkioxari, Piotr Doll{\'a}r, and Ross Girshick.
\newblock Mask r-cnn.
\newblock In {\em ICCV}, pages 2961--2969, 2017.

\bibitem{resnet}
Kaiming He, Xiangyu Zhang, Shaoqing Ren, and Jian Sun.
\newblock Deep residual learning for image recognition.
\newblock In {\em CVPR}, pages 770--778, 2016.

\bibitem{henaff2015deep}
Mikael Henaff, Joan Bruna, and Yann LeCun.
\newblock Deep convolutional networks on graph-structured data.
\newblock {\em arXiv preprint arXiv:1506.05163}, 2015.

\bibitem{gelu}
Dan Hendrycks and Kevin Gimpel.
\newblock Gaussian error linear units (gelus).
\newblock {\em arXiv preprint arXiv:1606.08415}, 2016.

\bibitem{hoffer2020augment}
Elad Hoffer, Tal Ben-Nun, Itay Hubara, Niv Giladi, Torsten Hoefler, and Daniel
  Soudry.
\newblock Augment your batch: Improving generalization through instance
  repetition.
\newblock In {\em CVPR}, 2020.

\bibitem{mobilenet}
Andrew~G Howard, Menglong Zhu, Bo~Chen, Dmitry Kalenichenko, Weijun Wang,
  Tobias Weyand, Marco Andreetto, and Hartwig Adam.
\newblock Mobilenets: Efficient convolutional neural networks for mobile vision
  applications.
\newblock {\em arXiv preprint arXiv:1704.04861}, 2017.

\bibitem{huang2016deep}
Gao Huang, Yu~Sun, Zhuang Liu, Daniel Sedra, and Kilian~Q Weinberger.
\newblock Deep networks with stochastic depth.
\newblock In {\em ECCV}, pages 646--661. Springer, 2016.

\bibitem{mindspore}
Huawei.
\newblock Mindspore.
\newblock \url{https://www.mindspore.cn/}, 2020.

\bibitem{jain2016structural}
Ashesh Jain, Amir~R Zamir, Silvio Savarese, and Ashutosh Saxena.
\newblock Structural-rnn: Deep learning on spatio-temporal graphs.
\newblock In {\em CVPR}, pages 5308--5317, 2016.

\bibitem{jing2022learning}
Yongcheng Jing, Yining Mao, Yiding Yang, Yibing Zhan, Mingli Song, Xinchao
  Wang, and Dacheng Tao.
\newblock Learning graph neural networks for image style transfer.
\newblock In {\em ECCV}, 2022.

\bibitem{kipf2016semi}
Thomas~N Kipf and Max Welling.
\newblock Semi-supervised classification with graph convolutional networks.
\newblock In {\em ICLR}, 2017.

\bibitem{alexnet}
Alex Krizhevsky, Ilya Sutskever, and Geoffrey~E Hinton.
\newblock Imagenet classification with deep convolutional neural networks.
\newblock In {\em NeurIPS}, pages 1097--1105, 2012.

\bibitem{landrieu2018large}
Loic Landrieu and Martin Simonovsky.
\newblock Large-scale point cloud semantic segmentation with superpoint graphs.
\newblock In {\em CVPR}, pages 4558--4567, 2018.

\bibitem{lecun1998gradient}
Yann LeCun, L{\'e}on Bottou, Yoshua Bengio, and Patrick Haffner.
\newblock Gradient-based learning applied to document recognition.
\newblock {\em Proceedings of the IEEE}, 86(11):2278--2324, 1998.

\bibitem{deepgcn}
Guohao Li, Matthias Muller, Ali Thabet, and Bernard Ghanem.
\newblock Deepgcns: Can gcns go as deep as cnns?
\newblock In {\em ICCV}, pages 9267--9276, 2019.

\bibitem{li2018deeper}
Qimai Li, Zhichao Han, and Xiao-Ming Wu.
\newblock Deeper insights into graph convolutional networks for semi-supervised
  learning.
\newblock In {\em AAAI}, pages 3538--3545, 2018.

\bibitem{asmlp}
Dongze Lian, Zehao Yu, Xing Sun, and Shenghua Gao.
\newblock As-mlp: An axial shifted mlp architecture for vision.
\newblock In {\em ICLR}, 2022.

\bibitem{retinanet}
Tsung-Yi Lin, Priya Goyal, Ross Girshick, Kaiming He, and Piotr Doll{\'a}r.
\newblock Focal loss for dense object detection.
\newblock In {\em ICCV}, 2017.

\bibitem{coco}
Tsung-Yi Lin, Michael Maire, Serge Belongie, James Hays, Pietro Perona, Deva
  Ramanan, Piotr Doll{\'a}r, and C~Lawrence Zitnick.
\newblock Microsoft coco: Common objects in context.
\newblock In {\em ECCV}, pages 740--755, 2014.

\bibitem{swin}
Ze~Liu, Yutong Lin, Yue Cao, Han Hu, Yixuan Wei, Zheng Zhang, Stephen Lin, and
  Baining Guo.
\newblock Swin transformer: Hierarchical vision transformer using shifted
  windows.
\newblock In {\em ICCV}, pages 10012--10022, 2021.

\bibitem{fcn}
Jonathan Long, Evan Shelhamer, and Trevor Darrell.
\newblock Fully convolutional networks for semantic segmentation.
\newblock In {\em CVPR}, 2015.

\bibitem{adamw}
Ilya Loshchilov and Frank Hutter.
\newblock Decoupled weight decay regularization.
\newblock {\em arXiv preprint arXiv:1711.05101}, 2017.

\bibitem{micheli2009neural}
Alessio Micheli.
\newblock Neural network for graphs: A contextual constructive approach.
\newblock {\em IEEE Transactions on Neural Networks}, 20(3):498--511, 2009.

\bibitem{niepert2016learning}
Mathias Niepert, Mohamed Ahmed, and Konstantin Kutzkov.
\newblock Learning convolutional neural networks for graphs.
\newblock In {\em ICML}, pages 2014--2023. PMLR, 2016.

\bibitem{oono2019graph}
Kenta Oono and Taiji Suzuki.
\newblock Graph neural networks exponentially lose expressive power for node
  classification.
\newblock In {\em ICLR}, 2020.

\bibitem{pytorch}
Adam Paszke, Sam Gross, Francisco Massa, Adam Lerer, James Bradbury, Gregory
  Chanan, Trevor Killeen, Zeming Lin, Natalia Gimelshein, Luca Antiga, et~al.
\newblock Pytorch: An imperative style, high-performance deep learning library.
\newblock {\em NeurIPS}, 2019.

\bibitem{fasterRCNN}
Shaoqing Ren, Kaiming He, Ross Girshick, and Jian Sun.
\newblock Faster r-cnn: Towards real-time object detection with region proposal
  networks.
\newblock In {\em NIPS}, pages 91--99, 2015.

\bibitem{imagenet}
Olga Russakovsky, Jia Deng, Hao Su, Jonathan Krause, Sanjeev Satheesh, Sean Ma,
  Zhiheng Huang, Andrej Karpathy, Aditya Khosla, Michael Bernstein, et~al.
\newblock Imagenet large scale visual recognition challenge.
\newblock {\em International Journal of Computer Vision}, 115(3):211--252,
  2015.

\bibitem{scarselli2008graph}
Franco Scarselli, Marco Gori, Ah~Chung Tsoi, Markus Hagenbuchner, and Gabriele
  Monfardini.
\newblock The graph neural network model.
\newblock {\em IEEE transactions on neural networks}, 20(1):61--80, 2008.

\bibitem{sen2008collective}
Prithviraj Sen, Galileo Namata, Mustafa Bilgic, Lise Getoor, Brian Galligher,
  and Tina Eliassi-Rad.
\newblock Collective classification in network data.
\newblock {\em AI magazine}, 29(3):93--93, 2008.

\bibitem{botnet}
Aravind Srinivas, Tsung-Yi Lin, Niki Parmar, Jonathon Shlens, Pieter Abbeel,
  and Ashish Vaswani.
\newblock Bottleneck transformers for visual recognition.
\newblock In {\em CVPR}, pages 16519--16529, 2021.

\bibitem{label-smooth}
Christian Szegedy, Vincent Vanhoucke, Sergey Ioffe, Jon Shlens, and Zbigniew
  Wojna.
\newblock Rethinking the inception architecture for computer vision.
\newblock In {\em CVPR}, 2016.

\bibitem{wavemlp}
Yehui Tang, Kai Han, Jianyuan Guo, Chang Xu, Yanxi Li, Chao Xu, and Yunhe Wang.
\newblock An image patch is a wave: Phase-aware vision mlp.
\newblock In {\em CVPR}, 2022.

\bibitem{mixer}
Ilya~O Tolstikhin, Neil Houlsby, Alexander Kolesnikov, Lucas Beyer, Xiaohua
  Zhai, Thomas Unterthiner, Jessica Yung, Andreas Steiner, Daniel Keysers,
  Jakob Uszkoreit, et~al.
\newblock Mlp-mixer: An all-mlp architecture for vision.
\newblock In {\em NeurIPS}, volume~34, 2021.

\bibitem{resmlp}
Hugo Touvron, Piotr Bojanowski, Mathilde Caron, Matthieu Cord, Alaaeldin
  El-Nouby, Edouard Grave, Gautier Izacard, Armand Joulin, Gabriel Synnaeve,
  Jakob Verbeek, et~al.
\newblock Resmlp: Feedforward networks for image classification with
  data-efficient training.
\newblock {\em arXiv preprint arXiv:2105.03404}, 2021.

\bibitem{deit}
Hugo Touvron, Matthieu Cord, Matthijs Douze, Francisco Massa, Alexandre
  Sablayrolles, and Herv{\'e} J{\'e}gou.
\newblock Training data-efficient image transformers \& distillation through
  attention.
\newblock In {\em ICML}, 2021.

\bibitem{convmixer}
Asher Trockman and J~Zico Kolter.
\newblock Patches are all you need?
\newblock {\em arXiv preprint arXiv:2201.09792}, 2022.

\bibitem{Att}
Ashish Vaswani, Noam Shazeer, Niki Parmar, Jakob Uszkoreit, Llion Jones,
  Aidan~N Gomez, {\L}ukasz Kaiser, and Illia Polosukhin.
\newblock Attention is all you need.
\newblock {\em NeurIPS}, 2017.

\bibitem{virmaux2018lipschitz}
Aladin Virmaux and Kevin Scaman.
\newblock Lipschitz regularity of deep neural networks: analysis and efficient
  estimation.
\newblock In {\em NeurIPS}, pages 3839--3848, 2018.

\bibitem{wale2008comparison}
Nikil Wale, Ian~A Watson, and George Karypis.
\newblock Comparison of descriptor spaces for chemical compound retrieval and
  classification.
\newblock {\em Knowledge and Information Systems}, 14(3):347--375, 2008.

\bibitem{wang2019learning}
Runzhong Wang, Junchi Yan, and Xiaokang Yang.
\newblock Learning combinatorial embedding networks for deep graph matching.
\newblock In {\em ICCV}, pages 3056--3065, 2019.

\bibitem{pvt}
Wenhai Wang, Enze Xie, Xiang Li, Deng-Ping Fan, Kaitao Song, Ding Liang, Tong
  Lu, Ping Luo, and Ling Shao.
\newblock Pyramid vision transformer: A versatile backbone for dense prediction
  without convolutions.
\newblock In {\em ICCV}, 2021.

\bibitem{edgeconv}
Yue Wang, Yongbin Sun, Ziwei Liu, Sanjay~E Sarma, Michael~M Bronstein, and
  Justin~M Solomon.
\newblock Dynamic graph cnn for learning on point clouds.
\newblock {\em Acm Transactions On Graphics (tog)}, 38(5):1--12, 2019.

\bibitem{wightman2021resnet}
Ross Wightman, Hugo Touvron, and Herv{\'e} J{\'e}gou.
\newblock Resnet strikes back: An improved training procedure in timm.
\newblock {\em arXiv preprint arXiv:2110.00476}, 2021.

\bibitem{cvt}
Haiping Wu, Bin Xiao, Noel Codella, Mengchen Liu, Xiyang Dai, Lu~Yuan, and Lei
  Zhang.
\newblock Cvt: Introducing convolutions to vision transformers.
\newblock In {\em ICCV}, pages 22--31, 2021.

\bibitem{wu2021rethinking}
Kan Wu, Houwen Peng, Minghao Chen, Jianlong Fu, and Hongyang Chao.
\newblock Rethinking and improving relative position encoding for vision
  transformer.
\newblock In {\em ICCV}, pages 10033--10041, 2021.

\bibitem{resnext}
Saining Xie, Ross Girshick, Piotr Doll{\'a}r, Zhuowen Tu, and Kaiming He.
\newblock Aggregated residual transformations for deep neural networks.
\newblock In {\em CVPR}, pages 1492--1500, 2017.

\bibitem{xu2017scene}
Danfei Xu, Yuke Zhu, Christopher~B Choy, and Li~Fei-Fei.
\newblock Scene graph generation by iterative message passing.
\newblock In {\em CVPR}, pages 5410--5419, 2017.

\bibitem{xu2018powerful}
Keyulu Xu, Weihua Hu, Jure Leskovec, and Stefanie Jegelka.
\newblock How powerful are graph neural networks?
\newblock In {\em ICLR}, 2018.

\bibitem{xu2021learning}
Yixing Xu, Kai Han, Chang Xu, Yehui Tang, Chunjing Xu, and Yunhe Wang.
\newblock Learning frequency domain approximation for binary neural networks.
\newblock In {\em NeurIPS}, volume~34, pages 25553--25565, 2021.

\bibitem{xu2019positive}
Yixing Xu, Yunhe Wang, Hanting Chen, Kai Han, Chunjing Xu, Dacheng Tao, and
  Chang Xu.
\newblock Positive-unlabeled compression on the cloud.
\newblock In {\em NeurIPS}, volume~32, 2019.

\bibitem{yan2018spatial}
Sijie Yan, Yuanjun Xiong, and Dahua Lin.
\newblock Spatial temporal graph convolutional networks for skeleton-based
  action recognition.
\newblock In {\em AAAI}, 2018.

\bibitem{yang2018graph}
Jianwei Yang, Jiasen Lu, Stefan Lee, Dhruv Batra, and Devi Parikh.
\newblock Graph r-cnn for scene graph generation.
\newblock In {\em ECCV}, pages 670--685, 2018.

\bibitem{yang2020distilling}
Yiding Yang, Jiayan Qiu, Mingli Song, Dacheng Tao, and Xinchao Wang.
\newblock Distilling knowledge from graph convolutional networks.
\newblock In {\em CVPR}, pages 7074--7083, 2020.

\bibitem{yang2020cars}
Zhaohui Yang, Yunhe Wang, Xinghao Chen, Boxin Shi, Chao Xu, Chunjing Xu,
  Qi~Tian, and Chang Xu.
\newblock Cars: Continuous evolution for efficient neural architecture search.
\newblock In {\em Proceedings of the IEEE/CVF Conference on Computer Vision and
  Pattern Recognition}, pages 1829--1838, 2020.

\bibitem{metaformer}
Weihao Yu, Mi~Luo, Pan Zhou, Chenyang Si, Yichen Zhou, Xinchao Wang, Jiashi
  Feng, and Shuicheng Yan.
\newblock Metaformer is actually what you need for vision.
\newblock In {\em CVPR}, 2022.

\bibitem{cutmix}
Sangdoo Yun, Dongyoon Han, Seong~Joon Oh, Sanghyuk Chun, Junsuk Choe, and
  Youngjoon Yoo.
\newblock Cutmix: Regularization strategy to train strong classifiers with
  localizable features.
\newblock In {\em ICCV}, 2019.

\bibitem{mixup}
Hongyi Zhang, Moustapha Cisse, Yann~N Dauphin, and David Lopez-Paz.
\newblock mixup: Beyond empirical risk minimization.
\newblock In {\em ICLR}, 2018.

\bibitem{erasing}
Zhun Zhong, Liang Zheng, Guoliang Kang, Shaozi Li, and Yi~Yang.
\newblock Random erasing data augmentation.
\newblock In {\em AAAI}, volume~34, pages 13001--13008, 2020.

\bibitem{nas1}
Barret Zoph and Quoc~V Le.
\newblock Neural architecture search with reinforcement learning.
\newblock In {\em ICLR}, 2017.

\end{thebibliography}
}

\section*{Checklist}


\begin{enumerate}

	\item For all authors...
	\begin{enumerate}
		\item Do the main claims made in the abstract and introduction accurately reflect the paper's contributions and scope?
		\answerYes{See Section 1.}
		\item Did you describe the limitations of your work?
		\answerYes{}
		\item Did you discuss any potential negative societal impacts of your work?
		\answerNo{No potential negative societal impacts.}
		\item Have you read the ethics review guidelines and ensured that your paper conforms to them?
		\answerYes{}
	\end{enumerate}

	\item If you are including theoretical results...
	\begin{enumerate}
		\item Did you state the full set of assumptions of all theoretical results?
		\answerNA{}
		\item Did you include complete proofs of all theoretical results?
		\answerNA{}
	\end{enumerate}

	\item If you ran experiments...
	\begin{enumerate}
		\item Did you include the code, data, and instructions needed to reproduce the main experimental results (either in the supplemental material or as a URL)?
		\answerYes{See Section~\ref{sec:dataset}.}
		\item Did you specify all the training details (e.g., data splits, hyperparameters, how they were chosen)?
		\answerYes{See Section~\ref{sec:dataset}.}
		\item Did you report error bars (e.g., with respect to the random seed after running experiments multiple times)?
		\answerNo{The common settings on ImageNet and COCO datasets.}
		\item Did you include the total amount of compute and the type of resources used (e.g., type of GPUs, internal cluster, or cloud provider)?
		\answerYes{See Section~\ref{sec:dataset}.}
	\end{enumerate}

	\item If you are using existing assets (e.g., code, data, models) or curating/releasing new assets...
	\begin{enumerate}
		\item If your work uses existing assets, did you cite the creators?
		\answerYes{See Section~\ref{sec:dataset}.}
		\item Did you mention the license of the assets?
		\answerYes{See Section~\ref{sec:dataset}.}
		\item Did you include any new assets either in the supplemental material or as a URL?
		\answerNo{}
		\item Did you discuss whether and how consent was obtained from people whose data you're using/curating?
		\answerNA{}
		\item Did you discuss whether the data you are using/curating contains personally identifiable information or offensive content?
		\answerNA{}
	\end{enumerate}

	\item If you used crowdsourcing or conducted research with human subjects...
	\begin{enumerate}
		\item Did you include the full text of instructions given to participants and screenshots, if applicable?
		\answerNA{}
		\item Did you describe any potential participant risks, with links to Institutional Review Board (IRB) approvals, if applicable?
		\answerNA{}
		\item Did you include the estimated hourly wage paid to participants and the total amount spent on participant compensation?
		\answerNA{}
	\end{enumerate}

\end{enumerate}

\appendix

\section{Appendix}

\subsection{Theoretical Analysis}
In our ViG block, we propose to increase feature diversity in nodes by utilizing more feature transformation such as FFN module. We show the empirical comparison between vanilla ResGCN and our ViG model in our paper. Here we make a simple theoretical analysis of the benefit of FFN module in ViG on increasing the feature diversity. Given the output features of graph convolution $X\in\mathbb{R}^{N\times D}$, the feature diversity~\cite{dong2021attention} is measured as
\begin{equation}
	\gamma(X) = \|X-\mathbf{1}\mathbf{x}^T\|, \quad\text{where}\quad \mathbf{x} = \arg\min_{\mathbf{x}}\|X-\mathbf{1}\mathbf{x}^T\|,
\end{equation}
where $\|\cdot\|$ is the $\ell_{1,\infty}$ norm of a matrix. By applying FFN module on the features, we have the following theorem.
\begin{myTheo}\label{theorem1}
	Given a FFN module, the diversity $\gamma(\text{FFN}(X))$ of its output features satisfies
	\begin{equation}
		\gamma(\text{FFN}(X))\leq \lambda\gamma(X),
	\end{equation}
	where $\lambda$ is the Lipschitz constant of FFN with respect to $p$-norm for $p\in[1,\infty]$.
\end{myTheo}
\begin{proof}
	The FFN includes weight matrix multiplication, bias addition and elementwise nonlinear function, which all preserve the constancy-across-rows property of $\text{FFN}(\mathbf{1}\mathbf{x}^T)$. Therefore, we have
	\begin{align*}
		\gamma(\text{FFN}(X)) &= \|\text{FFN}(X)-\mathbf{1}\mathbf{x'}^T\|_{p}\\
		&\leq \|\text{FFN}(X)-\text{FFN}(\mathbf{1}\mathbf{x}^T)\|_{p} \quad\quad \triangleright \text{FFN preserves constancy-across-rows.}\\
		&\leq \lambda\|X-\mathbf{1}\mathbf{x}^T\|_{p} \quad\quad\quad\quad\quad\quad \triangleright \text{Definition of Lipschitz constant.}\\
		&= \lambda\gamma(X),
	\end{align*}
\end{proof}

The Lipschitz constant of FFN is related to the norm of weight matrices and is usually much larger than 1~\cite{virmaux2018lipschitz}. Thus, the Theorem~\ref{theorem1} shows that introducing $\gamma(\text{FFN}(X))$ in our ViG block tends to improve the feature diversity in graph neural network.

\subsection{Pseudocode}
The proposed Vision GNN framework is easy to be implemented based on the commonly-used layers without introducing complex operations. The pseudocode of the core part, \ie, ViG block is shown in Algorithm~\ref{alg:code}.
\begin{breakablealgorithm}
\caption{PyTorch-like Code of ViG Block}
\label{alg:code}
\definecolor{codeblue}{rgb}{0.25,0.5,0.5}
\definecolor{codekw}{rgb}{0.85, 0.18, 0.50}
\lstset{
	backgroundcolor=\color{white},
	basicstyle=\fontsize{8pt}{8pt}\ttfamily\selectfont,
	columns=fullflexible,
	breaklines=true,
	captionpos=b,
	commentstyle=\fontsize{8pt}{8pt}\color{codeblue},
	keywordstyle=\fontsize{8pt}{8pt}\color{codekw},
}
\begin{lstlisting}[language=python]
import torch.nn as nn
from gcn_lib.dense.torch_vertex import DynConv2d
# gcn_lib is downloaded from https://github.com/lightaime/deep_gcns_torch

class GrapherModule(nn.Module):
  """Grapher module with graph conv and FC layers
  """
  def __init__(self, in_channels, hidden_channels, k=9, dilation=1, drop_path=0.0):
    super(GrapherModule, self).__init__()
    self.fc1 = nn.Sequential(
      nn.Conv2d(in_channels, in_channels, 1, stride=1, padding=0),
      nn.BatchNorm2d(in_channels),
    )
    self.graph_conv = nn.Sequential(
      DynConv2d(in_channels, hidden_channels, k, dilation, act=None),
      nn.BatchNorm2d(hidden_channels),
      nn.GELU(),
    )
    self.fc2 = nn.Sequential(
      nn.Conv2d(hidden_channels, in_channels, 1, stride=1, padding=0),
      nn.BatchNorm2d(in_channels),
    )
    self.drop_path = DropPath(drop_path) if drop_path > 0. else nn.Identity()

  def forward(self, x):
    B, C, H, W = x.shape
    x = x.reshape(B, C, -1, 1).contiguous()
    shortcut = x
    x = self.fc1(x)
    x = self.graph_conv(x)
    x = self.fc2(x)
    x = self.drop_path(x) + shortcut
    return x.reshape(B, C, H, W)

class FFNModule(nn.Module):
  """Feed-forward Network
  """
  def __init__(self, in_channels, hidden_channels, drop_path=0.0):
    super(FFNModule, self).__init__()
    self.fc1 = nn.Sequential(
      nn.Conv2d(in_channels, hidden_channels, 1, stride=1, padding=0),
      nn.BatchNorm2d(hidden_channels),
      nn.GELU()
    )
    self.fc2 = nn.Sequential(
      nn.Conv2d(hidden_channels, in_channels, 1, stride=1, padding=0),
      nn.BatchNorm2d(in_channels),
    )
    self.drop_path = DropPath(drop_path) if drop_path > 0. else nn.Identity()

  def forward(self, x):
    shortcut = x
    x = self.fc1(x)
    x = self.fc2(x)
    x = self.drop_path(x) + shortcut
    return x

class ViGBlock(nn.Module):
  """ViG block with Grapher and FFN modules
  """
  def __init__(self, channels, k, dilation, drop_path=0.0):
    super(ViGBlock, self).__init__()
    self.grapher = GrapherModule(channels, channels * 2, k, dilation, drop_path)
    self.ffn = FFNModule(channels, channels * 4, drop_path)

  def forward(self, x):
    x = self.grapher(x)
    x = self.ffn(x)
    return x
\end{lstlisting}
\end{breakablealgorithm}

\end{document}